\documentclass[letterpaper, 10 pt, conference]{ieeeconf}  %

\usepackage{graphicx} %
\usepackage{subcaption} 
\usepackage{hyperref}
\usepackage{graphicx}
\usepackage{algorithm}
\usepackage{algorithmic}
\usepackage{epstopdf}

\usepackage[usenames, dvipsnames]{color}

\usepackage{multirow}
\usepackage{hhline}

\usepackage{mathtools}

\newcommand{\calF}{F}

\newcommand{\calM}{{M}}
\newcommand{\Lap}{\text{Lap}}
\newcommand{\card}[1]{\texttt{card}(#1)}

\usepackage{xcolor}
\newcommand{\kcx}[1]{}
\newcommand{\kc}[1]{}

\newcommand{\shs}[1]{{\textcolor{blue}{SS: #1}}}
\newcommand{\rmk}[1]{}

\newcommand{\shsit}[1]{{\itshape}}

\def\calS{\mathcal{S}}
\def\calQ{\mathcal{Q}}
\def\calX{\kset{k}}
\def\MQM{\text{MQM}}

\newtheorem{theorem}{Theorem}[section]
\newtheorem{definition}[theorem]{Definition}
\newtheorem{lemma}[theorem]{Lemma}

\def\FA{\mathcal{F}_{\mathcal{A}}}
\def\FB{\mathcal{F}_{\mathcal{B}}}
\def\XA{{X}^{\mathcal{A}}}
\def\XB{{X}^{\mathcal{B}}}

\def\ZA{{Z}_{\mathcal{A}}}
\def\ZB{{Z}_{\mathcal{B}}}
\def\wA{{w}_{\mathcal{A}}}
\def\wB{{w}_{\mathcal{B}}}
\def\SA{{S}_{\mathcal{A}}}
\def\SB{{S}_{\mathcal{B}}}
\def\QA{{Q}_{\mathcal{A}}}
\def\QB{{Q}_{\mathcal{B}}}

\def\epsilonA{{\epsilon_{\mathcal{A}}}}
\def\epsilonB{{\epsilon_{\mathcal{B}}}}
\def\RQ{{R\cup Q}}
\def\inner{\text{mid}}
\def\out{\text{out}}
\def\MA{{M_A}}
\def\MB{{M_B}}
\def\mqm{\text{MQM}}
\def\Ni{{N\backslash\{i\}}}
\def\a{a}
\def\b{b}
\def\aone{a^1}
\def\bone{b^1}
\def\atwo{a^2}
\def\btwo{b^2}
\def\calA{\mathcal{A}}

\def\Pr{P}
\def\effect{max-influence}

\def\eT{e_{\Theta}}
\def\et{e_{\theta}}

\newcommand{\mypara}[1]{\medskip\noindent{\textbf{#1}}}

\def\mmqmapprox{\text{MQMApprox}}
\def\mmqmexact{\text{MQMExact}}

\newcommand{\kset}[1]{[#1]}

\usepackage{sidecap}

\newcommand{\Dinf}[2]{D_{\infty}\big(#1 ~\|~ #2\big)}

\title{\LARGE \bf Composition Properties of Inferential Privacy for Time-Series Data}

\author{
Shuang Song\\ University of California San Diego \\shs037@eng.ucsd.edu \and 
Kamalika Chaudhuri\\ University of California San Diego \\kamalika@eng.ucsd.edu}

\begin{document}

\maketitle
\thispagestyle{empty}
\pagestyle{empty}

\begin{abstract}
With the proliferation of mobile devices and the internet of things, developing principled solutions for privacy in time series applications has become increasingly important. While differential privacy is the gold standard for database privacy, many time series applications require a different kind of guarantee, and a number of recent works have used some form of inferential privacy to address these situations.

However, a major barrier to using inferential privacy in practice is its lack of graceful composition -- even if the same or related sensitive data is used in multiple releases that are safe individually, the combined release may have poor privacy properties. In this paper, we study composition properties of a form of inferential privacy called Pufferfish when applied to time-series data.  We show that while general Pufferfish mechanisms may not compose gracefully, a specific Pufferfish mechanism, called the Markov Quilt Mechanism, which was recently introduced by~\cite{ourpufferfish}, has strong composition properties comparable to that of pure differential privacy when applied to time series data.
\end{abstract}

\IEEEpeerreviewmaketitle

\section{Introduction}

With the proliferation of mobile devices and the internet of things, large amounts of time series data are being collected, stored and mined to draw inferences about the physical environment. Examples include activity recordings of elderly patients to determine the state of their health, power consumption data of residential and commercial buildings to predict power demand responses, location trajectories of users over time to deliver suitable advertisements, among many others. Much of this information is extremely sensitive -- activity recordings yield information about what the patient is doing all day, power consumption of a residence can reveal occupancy, and location trajectories can reveal activities of the subjects. It is therefore imperative to develop principled and rigorous solutions that address privacy in these kinds of time series applications.

The gold standard for privacy in database applications has long been differential privacy~\cite{DMNS06}; the typical setting is that each record corresponds to the private value of a single person, and the goal is to design algorithms that can compute functions such as classifiers and clusterings on the sensitive data, while hiding the participation of a single person. Differential privacy has many good properties, such as post-processing invariance and graceful composition, which have led to its high popularity and practical use over the years.

Unfortunately, many of the time-series applications described above require a different kind of privacy guarantee. Consider the physical activity monitoring application for example, where the goal is to hide activity at small time intervals while revealing long-term activity patterns. Here the entire dataset is about a single patient, and hence hiding their participation will not be useful. An alternative is {\em{entry differential privacy}}, which hides the inclusion of activity at any given single time point in the data; since activities at close-by time points are highly correlated, this will not prevent an adversary from inferring the activity at the hidden time. To address these issues, a number of recent works~\cite{ourpufferfish, NDSS16, GK16, Xiong16} have used the notion of {\em{inferential privacy}}, where the goal is to prevent an adversary who has some prior knowledge, from inferring the state of the time series at any particular time. 

A clean and elegant framework for inferential privacy is Pufferfish~\cite{KM12}, which is our privacy framework of choice. Pufferfish models a privacy problem through a triple $(\calS, \calQ, \Theta)$; here $\calS$ is a set of secrets, which is a set of potential facts that we may wish to hide. $\calQ$ is a set of tuples of the form $(s_i, s_j)$ where $s_i, s_j \in \calS$ which represent which pairs of secrets should be indistinguishable to an adversary. Finally, $\Theta$ is a set of distributions that can plausibly generate the data and describes prior beliefs of an adversary. A mechanism $\calA$ is said to satisfy $\epsilon$-Pufferfish privacy in the framework $(\calS, \calQ, \Theta)$ if an adversary's posterior odds of every pair of secrets $(s_i, s_j)$ in $\calQ$ is within a factor of $e^{\epsilon}$ of its prior odds. Pufferfish models the physical activity monitoring application as follows -- $\calS$ consists of elements of the form $s_t^a$, which represent {\em{patient has activity $a$ at time $t$}}, $\calQ$ consists of tuples of the form $(s_t^a, s_t^b)$ for all $t$ and all activity pairs $(a, b)$, and $\Theta$ consists of a set of Markov Chains that describe how activities transition across time.  

However, a major limitation of Pufferfish privacy is that except under very special conditions, it often does not compose gracefully -- even if the same or related sensitive data is used in multiple Pufferfish releases that are individually safe, the combined release may have poor privacy guarantees~\cite{KM12}. In many real applications, same or related data is often used across applications, and this forms a major barrier to the practical applicability of Pufferfish. 

In this paper, we study this question, and we show a number of composition results for Pufferfish privacy for time series applications in the framework described above. Our results look at two scenarios -- sequential and parallel composition; the first is when the same sensitive data is used across multiple computations, and the second is when disjoint sections of the Markov Chain are used in different computations. Note that while in differential privacy, composition in the second case is trivial, this does not apply to Pufferfish, as information about the state of one segment of a Markov Chain can leak information about a correlated segment. 

For {\em{sequential composition}}, we show that while in general we cannot expect any arbitrary Pufferfish mechanism to compose gracefully even for the time series framework described above, a specific mechanism, called the Markov Quilt Mechanism, that was recently introduced by~\cite{ourpufferfish}, does compose linearly, much like pure differential privacy. For {\em{parallel composition}}, we provide two results; first, we show a general result that applies to any Pufferfish mechanism in the framework described above and shows that the privacy guarantee obtained from two releases on two disjoint segments $A$ and $B$ of the Markov Chain is the worse of the two guarantees plus a correction factor that depends on the distance between $A$ and $B$ and properties of the chain. Second, we show that if the two segments of the chain are far enough, then, under some mild conditions, using a specific version of the Markov Quilt Mechanism can provide even better parallel composition guarantees, matching those of differential privacy.  Our results thus demonstrate that the Markov Quilt Mechanism and its versions have strong composition properties when applied to Markov Chains, thus motivating their use for real time-series applications.

\subsection{Related Work}

Since graceful composition is a critical property of any privacy definition, there has been a significant amount of work on differential privacy composition~\cite{DMNS06}, and it is known to compose rather gracefully. \cite{DMNS06} shows that pure differential privacy composes linearly under sequential composition; for parallel composition, the guarantees are even better, and the combined privacy guarantee is the worst of the guarantees offered by the individual releases. \cite{DRV10} shows that a variant of differential privacy, called approximate differential privacy, has even better sequential composition properties than pure differential privacy. Optimal composition guarantees for both pure and approximate differential privacy are established by~\cite{KOV}. Finally, \cite{Abadi17} provides a method for numerically calculating privacy guarantees obtained from composing a number of approximate differentially private mechanisms.

In contrast, little is known about the composition properties of inferential privacy. \cite{KM12} provides examples to show that Pufferfish may not sequentially compose, except in some very special cases. \cite{HMD13} shows that a specialized version of Pufferfish, called Blowfish, which is somewhat closer to differential privacy does have graceful sequential composition properties; however, Blowfish does not apply to time-series data. \cite{ourpufferfish} provides limited privacy guarantees for the Markov Quilt Mechanism under serial composition; however, these guarantees are worse than linear, and they only apply under much more stringent conditions -- namely, if all the mechanisms use the same active Markov Quilt.

\section{Preliminaries}
\subsection{Time Series Data and Markov Chains}
\label{sec:mc}

It is common to model time-series data as Markov chains.

\mypara{Example 1.}
Suppose we have data tracking the physical activity of a subject: $(X_1,X_2,\dots,X_T)$ where $X_t$ denotes activity (e.g, running, sitting, etc) of the subject at time $t$. Our goal is to provide the aggregate activity pattern of the subject by releasing (an approximate) histogram, while preventing an adversary from finding out what the subject was doing at a specific time $t$.

\mypara{Example 2.}
Suppose we have power consumption data for a house: $(X_1,X_2,\dots,X_T)$, where $X_t$ is the power level in Watts at time $t$. 
Our goal is to output a general power consumption pattern of the household by releasing (an approximate) histogram of the power levels, while preventing an adversary from inferring the power level at a specific time $t$;
specific power levels may be sensitive information, as the presence or absence of family members at a given time can be inferred with the power level.

\mypara{Markov Chains.}
Temporal correlation in this kind of time-series data is usually captured by a Markov chain $X_1 \rightarrow X_2 \rightarrow \ldots \rightarrow X_T$, where $\kset{k}$ represents all possible states and $X_t\in\kset{k}$ represents the state at time $t$.
In Example 1, $X_t$ represents the activity performed by the subject at time $t$ and $\kset{k}$ represents all possible activities.
In Example 2, $X_t$ represents the power level of the house at time $t$ and $\kset{k}$ represents all possible power levels. 
The transition from one state to another is determined by a transition matrix $P$, and state of $X_1$ is drawn from an initial distribution $q$.

\subsection{The Pufferfish Privacy Framework}

Pufferfish privacy framework captures the privacy in these examples. We next define it and specify how the examples mentioned fit in.

A Pufferfish framework is specified by three parameters -- a set of secret $\calS$, a set of secret pairs $\calQ$ and a set of data distributions $\Theta$. $\calS$ consists of possible facts about the data that need to be protected. $\calQ \subseteq \calS \times \calS$ is a set of secret pairs that we want to be indistinguishable. $\Theta$ is a set of distributions that can plausibly generate the data and captures the correlation among records; each $\theta\in\Theta$ represents one adversary's belief of the data. The goal of Pufferfish framework is to ensure indistinguishability of the secrets pairs in $\calQ$ under any belief in $\Theta$.  
Now we define Pufferfish privacy under the framework $(\calS,\calQ,\Theta)$.

\begin{definition}[Pufferfish Privacy]
A privacy mechanism $M$ is said to be $\epsilon$-Pufferfish private in a framework $(\calS, \calQ, \Theta)$ if for datasets $X \sim \theta$ where $\theta \in \Theta$, for all secret pairs $(s_i, s_j) \in \calQ$ and for all $w \in \text{Range}(M)$, we have
\begin{equation} \label{eqn:defpf}
e^{-\epsilon} \leq \frac{P_{M,\theta}(M(X)=w|s_i, \theta)}{P_{M,\theta}(M(X)=w|s_j, \theta)} \leq e^{\epsilon} 
\end{equation}
when $s_i$ and $s_j$ are such that $P(s_i|\theta) \neq 0$, $P(s_j | \theta) \neq 0.$
\end{definition}

\mypara{Pufferfish Framework for Time-Series Data:}
We can model the time-series data described in the previous section with the following Pufferfish framework. 

Let the database be a Markov chain $X = (X_1 \rightarrow X_2 \rightarrow \ldots \rightarrow X_T)$, where each $X_i$ lies in the state space $\kset{k}$. Such a Markov Chain may be fully described by a tuple $(q, P)$ where $q$ is an initial distribution and $P$ is a transition matrix.

Let $s^i_a$ denote the event that $X_i$ takes value $a\in\kset{k}$. The set of secrets is $\calS = \{ s^i_a\ : a \in \calX, i \in \kset{T}\}$, and the set of secret pairs is $\calQ = \{ (s^i_a, s^i_b): a, b \in\kset{k}, a\neq b, i \in \kset{T} \}$. 
Each $\theta = (q_\theta, P_\theta) \in \Theta$ represents a Markov chain of the above structure with transition matrix $P_\theta$ and initial distribution $q_\theta$.

In the first example, the state space represents the set of all possible activities and $s^i_a$ or $X_i = a$ represents the event that the subject is engaged in activity $a$ at time $i$. $\calQ$ indicates that we do not want the adversary to distinguish whether the subject is engaged in activity $a$ or $b$ at a given time.
In the second example, the state space represents the set of all possible power levels and $s^i_a$ or $X_i = a$ represents the event that the power level of the house is $a$ at time $i$. $\calQ$ indicates that we do not want the adversary to distinguish whether the house is at power level $a$ or $b$ at a given time.

\subsection{Notation}

We use $X$ with a lowercase subscript, for example, $X_i$, to denote a single node in the Markov chain, and $X$ with an uppercase subscript, for example, $X_A$, to denote a set of nodes in the Markov chain. For a set of nodes $X_A$ we use the notation $\card{X_A}$ to denote the number of nodes in $X_A$. 

For $I\subseteq \kset{T}$, we use $X^{I}$ to denote the subchain $\{X_i\}_{i\in I} \subseteq X$, and  
we use $\calS^{I}$ to denote $\{s^i_a\ : a \in \calX, i \in I\}$, $\calQ^{I}$ to denote $\{ (s^i_a, s^i_b): a, b \in\kset{k}, a\neq b, i \in I\}$.

\subsection{The Markov Quilt Mechanism}
\cite{ourpufferfish} proposes the Markov Quilt Mechanism (MQM). It can be used to achieve Pufferfish privacy in the case where $\Theta$ consists of Bayesian networks, of which Markov chains are special cases. We restate the algorithm and the corresponding definitions in this section.

To understand the main idea of \mqm, consider a Markov chain $\theta$.
Any two nodes in $\theta$ are correlated to a certain degree, which means releasing the state of one node potentially provides information on the state of the other. However, the amount of correlation between two nodes usually decays as the distance between them grows. 
Consider a node $X_i$ in the Markov chain. The nodes close to $X_i$ can be highly influenced by its state, while the nodes that are far away are almost independent.
Therefore to hide the effect of a node $X_i$ on the result of a query, \mqm\ adds noise that is roughly proportional to the number of nearby nodes, and uses a small correction term to account for the effect of the almost independent set.

To measure the amount of dependence, \cite{ourpufferfish} defines \effect. 
\begin{definition}[\effect] \label{def:effect}
The \effect\ of a variable $X_i$ on a set of variables $X_A$ under $\Theta$ is
\begin{align}\label{eqn:effect}
&e_{\Theta}(X_A | X_i) =\\& \sup_{\theta \in \Theta} \max_{a, b \in \calX}  \max_{x_A \in \calX^{\card{X_A}}} \log \frac{\Pr(X_A = x_A | X_i = a, \theta)}{\Pr(X_A = x_A | X_i = b, \theta)}.\nonumber
\end{align}
\end{definition}
A higher \effect\ means higher level of correlation between $X_i$ and $X_A$, and \effect\ becomes $0$ if $X_i$ and $X_A$ are independent. For simplicity, we would use $e_\theta$ to denote $e_{\{\theta\}}$.

In a Markov chain, the \effect\ can be calculated exactly given the transition matrix $P_{\theta}$ and initial distribution $q_{\theta}$. It can also be approximated using properties of the stationary distribution and eigen-gap of the transition matrix if the Markov chain is irreducible and aperiodic. \cite{ourpufferfish} shows the following upper bound of \effect.
\begin{lemma}\label{lem:effect_approx}
For an irreducible and aperiodic Markov chain described by $\theta=(q_\theta, P_\theta)$, let $P^*_\theta$ be the time reversal of $P_\theta$.
Let $\pi_\theta$ be the stationary distribution of $\theta$ and $\pi_\theta^{\min} = \min_{x\in \kset{k}} \pi_\theta(x)$
and let 
$g_\theta = \min \{1 - |\lambda| : P_\theta P^*_\theta x = \lambda x, |\lambda| < 1\}$ be the eigen-gap of $P_\theta P^*_\theta$.
If $\pi_\theta >0$, $g_\theta >0$ and $a, b \geq \frac{2\log(1/\pi_\theta^{\min})}{g_{\theta}}$,
then for $X_Q = \{X_{i-a}, X_{i+b}\}$, 
\begin{align}\label{eqn:effect_approx}
    &e_{\theta}(X_Q|X_i) \\
\leq& 2 \log \frac{\pi_\theta^{\min} + \exp(-g_\theta a/2)}{\pi_\theta^{\min} - \exp(-g_\theta a/2)} + \log \frac{\pi_\theta^{\min} + \exp(-g_\theta b/2)}{\pi_\theta^{\min} - \exp(-g_\theta b/2)}. \nonumber
\end{align}
\end{lemma}

To facilitate efficient search for an almost independent set, \cite{ourpufferfish} then defines a Markov Quilt which takes into account the structure a Markov chain.
\begin{definition}[Markov Quilt]
\label{def:mquilt}
A set of nodes $X_Q$, $Q \subset \kset{n}$ in a Markov chain $X$ is a Markov Quilt for a node $X_i$ if the following conditions hold:
\begin{enumerate}
\item Deleting $X_Q$ partitions $X$ into parts $X_N$ and $X_R$ such that $X = X_N \cup X_Q \cup X_R$ and $X_i \in X_N$.
\item For all $x_R \in \calX^{\card{X_R}}$, all $x_Q \in \calX^{\card{X_Q}}$ and for all $a \in \calX$, $P(X_R = x_R | X_Q = x_Q, X_i = a) = P(X_R = x_R | X_Q = x_Q)$.
\end{enumerate}
Thus, $X_R$ is independent of $X_i$ conditioned on $X_Q$.
\end{definition}
Intuitively, $X_R$ is a set of ``remote" nodes that are far from $X_i$, and $X_N$ is the set of ``nearby" nodes; $X_N$ and $X_R$ are separated by the Markov Quilt $X_Q$. 
A Markov Quilt $X_Q$ (with corresponding $X_N$ and $X_R$) of $X_i$ is minimal if among all other Markov Quilts with the same nearby set $X_N$, it has the minimal cardinality.

\begin{lemma}\label{lem:mc minimal quilt}
In a Markov chain $X = \{X_k\}_{k=1}^T$, the set of minimal Markov Quilts of a node $X_i$ is
\begin{align}\label{eqn:minimal quilt}
S_{Q, i} = \{  \{ X_{i-a}, X_{i+b} \}, \{ X_{i-a} \}, \{X_{i+b}\}, \emptyset \nonumber\\
 | 1 \leq a \leq i - 1, 1 \leq b \leq T - i \}.
\end{align}
\end{lemma}

That is, one node on its left and one node to its right can form a Markov Quilt for $X_i$. Additionally, a Markov Quilt can also be formed by only one node $X_{i-a}$ (or $X_{i+b}$), in which case  $X_N = \{X_j\}_{j=i-a+1}^{T}$ (or $\{X_j\}_{j=1}^{i+b-1}$); and the empty Markov Quilt is also allowed, with corresponding $X_N$ as the whole chain and $X_R$ as the empty set.

\begin{algorithm}
\caption{\mqm (Dataset $D$, $1$-Lipschitz query $F$, $\Theta$, privacy parameter $\epsilon$)}
\label{alg:mqmExact}
\begin{algorithmic}
\FOR {all $\theta \in \Theta$}
\FOR {all $X_i$}
 \FOR {all Markov Quilts $X_Q \in S_{Q,i}$ where $S_{Q,i}$ is in \ref{eqn:minimal quilt}}
     \STATE {Calculate $e_{\{\theta\}}(X_Q | X_i)$}
    \IF {$e_{\{\theta\}}(X_Q | X_i) < \epsilon$} 
        \STATE{$\sigma_i^\theta(X_Q) = \frac{\card{X_N}}{\epsilon - e_{\{\theta\}}(X_Q | X_i)}$ \qquad /*\textbf{score of $X_Q$}*/}
    \ELSE 
        \STATE{$\sigma_i^\theta(X_Q) = \infty$}
    \ENDIF
\ENDFOR 
\STATE {$\sigma_i^\theta = \min_{X_Q \in S_{Q,i}} \sigma_i^\theta(X_Q)$}  
\ENDFOR
\STATE {$\sigma_{\max}^\theta = \max_i \sigma_i^\theta$}
\ENDFOR
\STATE {$\sigma_{\max} = \max_{\theta\in\Theta} \sigma_{\max}^\theta$}
\STATE{\textbf{return} $F(D) + \sigma_{\max} \cdot Z$, where $Z \sim \Lap(1)$}
\end{algorithmic}%
\end{algorithm}

The Markov Quilt Mechanism for Markov chain is restated in Algorithm \ref{alg:mqmExact}.
Intuitively, for each node $X_i$, \mqm\ searches over all the Markov Quilts, finds the one with the least amount of noise needed, and finally adds the noise that is sufficient to protect privacy of all nodes.

It was shown in \cite{ourpufferfish} that \mqm\ guarantees $\epsilon$-Pufferfish privacy in the framework $(\calS, \calQ, \Theta)$ in Section~\ref{sec:mc} provided that the query $F$ in Algorithm 
\ref{alg:mqmExact} is $1$-Lipschitz. 
Note that any Lipschitz function can be scaled to $1$-Lipschitz function.

Observe that Algorithm~\ref{alg:mqmExact} does not specify how to compute \effect. \cite{ourpufferfish} proposes two versions of \mqm\ -- \mmqmexact\ which computes the exact \effect\ using Definition~\ref{def:effect}, and \mmqmapprox\ which computes an upper bound of \effect\ using Lemma~\ref{lem:effect_approx}.

\mypara{Previous Results on Composition}

To design more sophisticated privacy preserving algorithms, we need to understand the privacy guarantee of the combination of two private algorithms, which is called composition.

There are two types of composition -- parallel and sequential. The first describes the case where multiple privacy algorithms are applied on disjoint data sets, while the second describes the case where they are applied to the same data.

A major advantage of differential privacy is that it composes gracefully. \cite{DMNS06} shows that applying $K$ differentially private algorithms, each with $\epsilon_k$-differential privacy, guarantees $\max_k \epsilon_k$-differential privacy under parallel composition, and $\sum_k \epsilon_k$-differential privacy under sequential composition. 
Better and more sophisticated composition results have been shown for approximate differential privacy \cite{DRV10} \cite{KOV}.

Unlike differential privacy, Pufferfish privacy does not always compose linearly~\cite{KM12}. However, we can still hope to achieve composition for special Pufferfish mechanisms 
or for special classes of data distributions $\Theta$.

\cite{ourpufferfish} does not provide any parallel composition result. 
The following sequential composition result for \mqm\ on Markov chain is provided.
\begin{theorem}
\label{thm:old_compose}
Let $\{F_k\}_{k=1}^K$ be a set of Lipschitz queries, $(\calS,\calQ,\Theta)$ be a Pufferfish framework as defined in Section~\ref{sec:mc}, and $D$ be a database. Given fixed Markov Quilt sets $\{S_{Q,i}\}_{i=1}^n$ for all $X_i$, let $M_k(D)$ denote the Markov Quilt Mechanism that releases $F_k(D)$ with $\epsilon_k$-Pufferfish privacy under $(\calS,\calQ,\Theta)$ using Markov Quilt sets $\{S_{Q,i}\}_{i=1}^n$. Then releasing $(M_1(D), \dots,M_K(D))$
guarantees $K\max_{k\in\kset{K}} \epsilon_k$-Pufferfish privacy under $(\calS,\calQ,\Theta)$.
\end{theorem}

Notice that this result holds only when the same Markov Quilts are used for all releases. Moreover, the final privacy guarantee depends on the worst privacy guarantees $\max_k \epsilon_k$ over the $K$ releases. In practice, it might not be easy to enforce the \mqm\ to use the same Markov Quilts at all releases; and if even one of the releases guarantees large $\epsilon_k$, the final privacy guarantee can be bad.

\section{Results}
As discussed in the previous section, general Pufferfish mechanisms do not compose linearly. However, we can exploit the properties of data distributions -- Markov chains, and properties of the specific Pufferfish mechanism -- \mqm\ to obtain new parallel composition result as well as improved sequential composition result.

\subsection{Parallel Composition}
\mypara{Setup:}
Consider the Pufferfish framework $(\calS,\calQ,\Theta)$ as described in Section~\ref{sec:mc}.
Parallel composition can be formulated as follows.

Suppose there are two subchains of the Markov chain, $\XA=X^{[T_1, T_2]}$ and $\XB=X^{[T_3, T_4]}$
where $1\leq T_1<T_2<T_3<T_4\leq T$; 
and correspondingly, let $\SA=\calS^{[T_1,T_2]}$, $\QA=\calQ^{[T_1,T_2]}$ and $\SB=\calS^{[T_3,T_4]}$, $\QB=\calQ^{[T_3,T_4]}$.

Suppose Alice has access to subchain $\XA$ and wants to release Lipchitz query $\FA$ while guaranteeing $\epsilonA$-Pufferfish Privacy under framework $(\SA,\QA,\Theta)$; 
and Bob has access to $\XB$ and wants to release Lipchitz query $\FB$ while guaranteeing $\epsilonB$-Pufferfish Privacy under framework $(\SB,\QB,\Theta)$.

Our goal is to determine how strong the Pufferfish privacy guarantee we can get for releasing $(\MA(\XA), \MB(\XB))$.

\mypara{A General Result for Markov Chains}
\begin{theorem}\label{thm:para_any_pf}
Suppose $\MA, \MB$ are two mechanisms such that $\MA(\XA)$ guarantees $\epsilonA$-Pufferfish privacy under framework $(\SA,\QA,\Theta)$ and $\MB(\XB)$ guarantees $\epsilonB$-Pufferfish privacy under framework $(\SB,\QB,\Theta)$. Then releasing $(\MA(\XA), \MB(\XB))$ guarantees $\max\{\min\{\epsilonA + \epsilonB, \epsilonA + \eT(X_{T_2} | X_{T_3})\},
\min\{\epsilonB+\epsilonA, \epsilonB + \eT(X_{T_3}|X_{T_2})\}\}$-Pufferfish Privacy under framework $(\SA\cup\SB,\QA\cup\QB,\Theta)$.
\end{theorem}

Comparing with parallel composition for differential privacy, here we have the extra terms $\eT(X_{T_2} | X_{T_3})$ and $\eT(X_{T_3} | X_{T_2})$ which capture the correlation between $X_{T_2}$ and $X_{T_3}$ -- the end point of the first subchain and the starting point of the second. This is to be expected, since there is correlation among states in the Markov chain. Intuitively, if the two subchains are close enough, releasing information on one can cause a privacy breach of the other. 

\mypara{\mqm\ on Markov Chains}

Let $\MQM(D,F,\epsilon,(\calS,\calQ,\Theta))$ denote the output of \mqm\ on dataset $D$, query function $F$, privacy parameter $\epsilon$ and Pufferfish framework $(\calS,\calQ,\Theta)$.
Suppose Alice and Bob use \mmqmapprox\ to publish $\MQM(\XA,\FA,\epsilonA,(\SA,\QA,\Theta))$ and $\MQM(\XB,\FB,\epsilonB,(\SB,\QB,\Theta))$ respectively.

Before we establish a parallel composition result, we begin with a definition.

\begin{definition}(Active Markov Quilt)
Consider an instance of the Markov Quilt Mechanism $M$. We say that a Markov Quilt $X_Q$ (with corresponding $X_N, X_R$) for a node $X_i$ is {\em{active}} with respect to $\theta\in\Theta$ if
$X_Q = \arg\min_{X_Q \in S_{Q,i}} \sigma_i^{\theta}(X_Q)$, and thus $\sigma_i^{\theta}(X_Q) = \sigma_i^{\theta}$.
\end{definition}

\begin{theorem}\label{thm:para_mqm}
Suppose we run \mmqmapprox\ to release $(\MQM(\XA,\FA,\epsilonA, (\SA,\QA,\Theta)), \MQM(\XB,\FB,\epsilonB,(\SB,\\ \QB,\Theta)))$. 
If the following conditions hold:
\begin{enumerate}
\item for any $\theta\in\Theta$, there exists some $X_i\in\XA$ and $X_j\in\XB$ such that the active Markov Quilts of $X_i$ and $X_j$ with respect to $\theta$ are of the form $\{X_{i-a}, X_{i+b}\}$ and $\{X_{j-a'}, X_{j+b'}\}$ respectively for some $a$, $b$, $a'$, $b'$, and
\item $T_3 - T_2 \geq \max\{T_2 - T_1, T_4-T_3\}$, i.e., $\XA,\XB$ are far from each other compared to their lengths,
\end{enumerate}
then the release guarantees $\max(\epsilonA,\epsilonB)$-Pufferfish Privacy under the framework $(\SA\cup\SB,\QA\cup\QB,\Theta)$.
\end{theorem}

The main intuition is as follows. Note that we require the active Markov Quilt of some $X_i \in \XA$ to be of the form $\{X_{i-a},X_{i+b} \}$. For any $X_{i'} \in \XA$, the correction factor added to account for the effect of the nodes $\{X_k \in \XA\}_{k\geq i+b}$ also automatically accounts for the effect of $\XB = X^{[T_3,T_4]}$, provided that $[T_3, T_4]$ does not overlap with $[i'-a,i'+b]$. This is ensured by the second condition in Theorem~\ref{thm:para_mqm}. 

\subsection{Sequential Composition}
Consider the case when Alice and Bob have access to the entire Markov Chain $X=\{X_t\}_{t=1}^T$, and want to publish Lipschitz queries $\calF_A,\calF_B$ with Pufferfish parameters $\epsilon_A,\epsilon_B$ and Pufferfish framework $(\calS,\calQ,\Theta)$ as described in Section~\ref{sec:mc}. 

\mypara{General Results for Markov Chains}

First, we show that an arbitrary Pufferfish mechanism does not compose linearly even when $\Theta$ consists of Markov chains.

\begin{theorem}\label{thm:sequential_any_pf_counterex}
There exists a Markov chain $X$, a function $\calF$ and mechanisms $\MA$, $\MB$ such that both $\MA(X)$ and $\MB(X)$ guarantee $\epsilon$-Pufferfish privacy under framework $(\calS,\calQ,\Theta)$, yet releasing $(\MA(X),\MB(X))$ does not guarantee $2\epsilon$-Pufferfish privacy under framework $(\calS,\calQ,\Theta)$.
\end{theorem}

Now we show that arbitrary Pufferfish mechanisms compose with a correction factor that depends on the max-divergence between the joint and product distributions of $\MA(X)$ and $\MB(X)$. We define max-divergence first.

\begin{definition}[max-divergence]
Let $p$ and $q$ be two distributions with the same support. The max-divergence $D_{\infty}(p, q)$ between them is defined as: 
\[ \Dinf{X}{Y} = \sup_{x \in \text{support}(p)} \log \frac{p(x)}{q(x)}. \]
\end{definition}

Now we state the composition theorem.
\begin{theorem}\label{thm:sequential_any_pf}
Suppose $\MA, \MB$ are two mechanisms used by Alice and Bob 
which guarantee $\epsilonA$ and $\epsilonB$-Pufferfish privacy respectively under framework $(\calS,\calQ,\Theta)$.
If there exists $E$, such that for all ${s_i \in \calS}, \theta \in \Theta$,
\begin{align*}
\Dinf{ & P(\MA(X), \MB(X) | s_i, \theta)}{ \\
 & P(\MA(X) | s_i, \theta) P(\MB(X) | s_i, \theta)} \leq E \\
\Dinf{& P(\MA(X) | s_i, \theta) P(\MB(X) | s_i, \theta)}{\\
 & P(\MA(X), \MB(X) | s_i, \theta)} \leq E,
\end{align*}
then the releasing $(\MA(X), \MB(X))$ guarantees $(\epsilonA + \epsilonB + 2 E)$-Pufferfish Privacy under framework $(\calS,\calQ,\Theta)$.
\end{theorem}

The max-divergence between the joint and product distributions of $\MA(X)$ and $\MB(X)$ measures the amount of dependence between the two releases. The more independent they are, the smaller the max-divergence would be and the stronger privacy the algorithm guarantees.

\mypara{\mqm\ on Markov Chains}

We next show that we can further exploit the properties of \mqm\ to provide tighter privacy guarantees than that provided in \cite{ourpufferfish}.
 
Suppose Alice and Bob use \mqm\ to achieve Pufferfish privacy under the same framework $(\calS,\calQ,\Theta)$. We show that when $\Theta$ consists of Markov chains, even if the two runs of \mqm\ use different Markov Quilts, \mqm\ still compose linearly.
This result applies to both \mmqmexact\ and \mmqmapprox. 

\begin{theorem}\label{thm:sequential_mqm_chain}
For the Pufferfish framework $(\calS,\calQ,\Theta)$ defined in Section~\ref{sec:mc}, 
releasing $\MQM(X,\calF_k,\epsilon_k,(\calS,\calQ,\Theta))$ for all $k\in\kset{K}$ guarantees $\sum_{k\in\kset{K}} \epsilon_k$-Pufferfish privacy under framework $(\calS,\calQ,\Theta)$.
\end{theorem}

This result shows that \mqm\ on Markov chain achieves the same composition guarantee as pure differential privacy. 
Comparing to the composition results provided in \cite{ourpufferfish}, i.e., Theorem~\ref{thm:old_compose}, Theorem~\ref{thm:sequential_mqm_chain} provides better privacy guarantee under less restricted conditions.
It does not require the same Markov Quilts to be used in the two runs of \mqm. 
Moreover, the privacy guarantee is better when $\epsilon_k$'s are different -- $\sum_k \epsilon_k$ as opposite to $K\max_k \epsilon_k$.

\section{Conclusion}
In conclusion, motivated by emerging sensing applications, we study composition properties of Pufferfish, a form of inferential privacy, for certain kinds of time-series data. We provide both sequential and parallel composition results. Our results illustrate that while Pufferfish does not have strong composition properties in general, variants of the recently introduced Markov Quilt Mechanism that guarantees Pufferfish privacy for time series data, do compose well, and have strong composition properties comparable to pure differential privacy. We believe that these results make these mechanisms attractive for practical time series applications. 

\section*{Acknowledgment}
We thank Joseph Geumlek, Sewoong Oh and Yizhen Wang for initial discussions. This work was partially supported by NSF under IIS 1253942, ONR under N00014-16-1-2616 and a Google Faculty Research Award.

\bibliographystyle{plain}
\bibliography{ref}

\begin{thebibliography}{10}

\bibitem{Abadi17}
Mart{\'\i}n Abadi, Andy Chu, Ian Goodfellow, H~Brendan McMahan, Ilya Mironov,
  Kunal Talwar, and Li~Zhang.
\newblock Deep learning with differential privacy.
\newblock In {\em Proceedings of the 2016 ACM SIGSAC Conference on Computer and
  Communications Security}, pages 308--318. ACM, 2016.

\bibitem{DMNS06}
C.~Dwork, F.~McSherry, K.~Nissim, and A.~Smith.
\newblock Calibrating noise to sensitivity in private data analysis.
\newblock In {\em Theory of Cryptography}, 2006.

\bibitem{DRV10}
Cynthia Dwork, Guy~N Rothblum, and Salil Vadhan.
\newblock Boosting and differential privacy.
\newblock In {\em Foundations of Computer Science (FOCS), 2010 51st Annual IEEE
  Symposium on}, pages 51--60. IEEE, 2010.

\bibitem{GK16}
Arpita Ghosh and Robert Kleinberg.
\newblock Inferential privacy guarantees for differentially private mechanisms.
\newblock {\em arXiv preprint arXiv:1603.01508}, 2016.

\bibitem{HMD13}
Xi~He, Ashwin Machanavajjhala, and Bolin Ding.
\newblock Blowfish privacy: tuning privacy-utility trade-offs using policies.
\newblock In {\em {SIGMOD} '14}, pages 1447--1458, 2014.

\bibitem{KOV}
Peter Kairouz, Sewoong Oh, and Pramod Viswanath.
\newblock The composition theorem for differential privacy.
\newblock In Francis Bach and David Blei, editors, {\em Proceedings of the 32nd
  International Conference on Machine Learning}, volume~37 of {\em Proceedings
  of Machine Learning Research}, pages 1376--1385, Lille, France, 07--09 Jul
  2015. PMLR.

\bibitem{KM12}
Daniel Kifer and Ashwin Machanavajjhala.
\newblock Pufferfish: {A} framework for mathematical privacy definitions.
\newblock {\em {ACM} Trans. Database Syst.}, 39(1):3, 2014.

\bibitem{NDSS16}
Changchang Liu, Supriyo Chakraborty, and Prateek Mittal.
\newblock Dependence makes you vulnerable: Differential privacy under dependent
  tuples.
\newblock In {\em NDSS 2016}, 2016.

\bibitem{ourpufferfish}
Shuang Song, Yizhen Wang, and Kamalika Chaudhuri.
\newblock Pufferfish privacy mechanisms for correlated data.
\newblock In {\em Proceedings of the 2017 ACM International Conference on
  Management of Data}, pages 1291--1306. ACM, 2017.

\bibitem{Xiong16}
Yonghui Xiao and Li~Xiong.
\newblock Protecting locations with differential privacy under temporal
  correlations.
\newblock In {\em Proceedings of the 22nd ACM SIGSAC CCS}.

\end{thebibliography}

\appendix
\section{Proofs of Composition Results}
\subsection{Proofs for Parallel Composition Results}
\begin{proof}(of Theorem~\ref{thm:para_any_pf})
Consider the case when the secret pair is $(X_{T_2}=a,X_{T_2}=b)$ for some $a,b$. For any $\theta \in \Theta$, we have
\begin{align*}
&	\frac{p(\MA(\XA)=\wA,\MB(\XB)=\wB | X_{T_2} = a,\theta)}{p(\MA(\XA)=\wA,\MB(\XB)=\wB | X_{T_2} = b,\theta)} \\
= \ &	\frac{p(\MA(\XA)=\wA| X_{T_2} = a,\theta)}{p(\MA(\XA)=\wA| X_{T_2} = b,\theta)} \\
&\times
\frac{p(\MB(\XB)=\wB| X_{T_2} = a,\theta)}{p(\MB(\XB)=\wB| X_{T_2} = b,\theta)}
\end{align*}
since $\XA$ and $\XB$ are independent conditioned on $X_{T_2}$. The first ratio is upper bounded by $e^\epsilonA$ since $\MA$ is $\epsilonA$-Pufferfish private.
The second ratio can be written as
\begin{align}\label{eqn:comp parallel general wB}
& \frac{p(\MB(\XB)=\wB| X_{T_2} = a,\theta)}{p(\MB(\XB)=\wB| X_{T_2} = b,\theta)}\nonumber\\
=&	\frac{\int_{} p(\MB(\XB)=\wB,X_{T_3}=x_{T_3} | X_{T_2} = a,\theta) d {x_{T_3}}}{\int_{} p(\MB(\XB)=\wB,X_{T_3}=x_{T_3} | X_{T_2} = b,\theta) d {x_{T_3}}} \nonumber \\
=&	\frac{\splitfrac{\int_{} p(\MB(\XB)=\wB|X_{T_3}=x_{T_3},\theta)}{ p(X_{T_3}=x_{T_3} | X_{T_2} = a,\theta) d {x_{T_3}}}}{\splitfrac{\int_{} p(\MB(\XB)=\wB|X_{T_3}=x_{T_3},\theta)}{ p(X_{T_3}=x_{T_3} | X_{T_2} = b,\theta) d {x_{T_3}}}}
\end{align}
where the second equality follows from the fact that $\XB$ is independent of $X_{T_2}$ given $X_{T_3}$. \\
Since $\max_{a,b,x_{T_3}}\frac{p(X_{T_3}=x_{T_3} | X_{T_2} = a,\theta)}{p(X_{T_3}=x_{T_3} | X_{T_2} = b,\theta)} \leq e^{e_{\theta}(X_{T_3} | X_{T_2})}$ and \\ $\et(X_{T_3}|X_{T_2})\leq \eT(X_{T_3}|X_{T_2})$, \eqref{eqn:comp parallel general wB} can be upper bounded by 
$$e^{\eT(X_{T_3} | X_{T_2})}.$$

On the other hand, \eqref{eqn:comp parallel general wB} is also upper bounded by 
\begin{align*}
&	\frac{\splitdfrac{\max_{x_{T_3}} p(\MB(\XB)=\wB|X_{T_3}=x_{T_3},\theta)}{ \int_{} p(X_{T_3}=x_{T_3} | X_{T_2} = a,\theta) d {x_{T_3}}}}{\splitdfrac{\min_{x_{T_3}} p(\MB(\XB)=\wB|X_{T_3}=x_{T_3},\theta)}{ \int_{} p(X_{T_3}=x_{T_3} | X_{T_2} = b,\theta) d {x_{T_3}}}} \\
=&	\frac{\max_{x_{T_3}} p(\MB(\XB)=\wB|X_{T_3}=x_{T_3},\theta)}{\min_{x_{T_3}} p(\MB(\XB)=\wB|X_{T_3}=x_{T_3},\theta)} 
\leq e^\epsilonB,
\end{align*} 
where the equality follows becasue $\int_{} p(X_{T_3}=x_{T_3} | X_{T_2}=x_{T_2},\theta) d {x_{T_3}} = 1$ for any $x_{T_2}$. Therefore \eqref{eqn:comp parallel general wB} is upper bounded by $\min\{e^{\eT(X_{T_3} | X_{T_2})}, e^{\epsilonB}\}$.

Combining the bound of the first ratio, we get 
\begin{align*}
&\frac{p(\MA(\XA)=\wA,\MB(\XB)=\wB | X_{T_2} = a,\theta)}{p(\MA(\XA)=\wA,\MB(\XB)=\wB | X_{T_2} = b,\theta)} \\
\leq& \min\{e^{\epsilonA + \eT(X_{T_3} | X_{T_2})}, e^{\epsilonA + \epsilonB}\}.
\end{align*}
If the secret pair is $(X_{i}=a,X_{i}=b,\theta)$ for some $a,b$ where $T_1 \leq i < T_2$, we have
\begin{align*}
&	p(\MA(\XA)=\wA,\MB(\XB)=\wB | X_i = a,\theta)\\
=&	\int_{} \ p(\MA(\XA)=\wA,\MB(\XB)=\wB,X_{T_2}=x_{T_2} | \\
& X_i = a,\theta) d{x_{T_2}} \\
=&	\int_{} \ p(\MA(\XA)=\wA,\MB(\XB)=\wB, X_i = a|X_{T_2}=\\
& x_{T_2} ,\theta) p(X_{T_2}=x_{T_2},\theta) / p(X_i=a,\theta) d{x_{T_2}} \\
=&	\int_{} \ p(\MA(\XA)=\wA, X_i = a|X_{T_2}=x_{T_2},\theta)  p(\MB(\XB)=\\
& \wB| X_{T_2}=x_{T_2},\theta) p(X_{T_2}=x_{T_2},\theta) / p(X_i=a,\theta) d{x_{T_2}} \\
=&	\int_{} \ p(\MA(\XA)=\wA, X_{T_2}=x_{T_2}|X_i = a,\theta) \\
& p(\MB(\XB)=\wB|X_{T_2}=x_{T_2},\theta) d{x_{T_2}},
\end{align*}
where the third equality is because $\XA,X_i$ are independent of $\XB$ given $X_{T_2}$.\\
Therefore we have
\begin{align*}
& \frac{p(\MA(\XA)=\wA,\MB(\XB)=\wB | X_i = a,\theta)}{p(\MA(\XA)=\wA,\MB(\XB)=\wB | X_i = b,\theta)}\\
= &		\frac{\splitdfrac{\int_{} p(\MA(\XA)=\wA, X_{T_2}=x_{T_2}|X_i = a,\theta)}{ p(\MB(\XB)=\wB|X_{T_2}=x_{T_2},\theta) d{x_{T_2}}}}{\splitdfrac{\int_{} p(\MA(\XA)=\wA, X_{T_2}=x_{T_2}|X_i = b,\theta)}{ p(\MB(\XB)=\wB|X_{T_2}=x_{T_2},\theta) d{x_{T_2}}}}\\
\leq &
        \frac{\max_{x_{T_2}}p(\MB(\XB)=\wB|X_{T_2}=x_{T_2},\theta)}{\min_{x_{T_2}}p(\MB(\XB)=\wB|X_{T_2}=x_{T_2},\theta)} \\
    &    \frac{\int_{} p(\MA(\XA)=\wA, X_{T_2}=x_{T_2}|X_i = a,\theta) d{x_{T_2}}}{\int_{} p(\MA(\XA)=\wA, X_{T_2}=x_{T_2}|X_i = b,\theta) d{x_{T_2}}}\\
\leq &	
        \frac{\max_{x_{T_2}}p(\MB(\XB)=\wB|X_{T_2}=x_{T_2},\theta)}{\min_{x_{T_2}}p(\MB(\XB)=\wB|X_{T_2}=x_{T_2},\theta)} \\
    &    \frac{p(\MA(\XA)=\wA|X_i = a,\theta)}{p(\MA(\XA)=\wA|X_i = b,\theta)}\\
\leq & 	\min\{e^{\epsilonA + \eT(X_{T_3} | X_{T_2})}, e^{\epsilonA + \epsilonB}\},
\end{align*}
where the last step follows from our previous bound for \eqref{eqn:comp parallel general wB} and the fact that $\MA$ guarantees $\epsilonA$ Pufferfish privacy.

The same analysis can be applied to the case where the secret is $(X_i=a,X_i=b,\theta)$ for some $T_3\leq i\leq T_4$ and the upper bound is $\min\{e^{\epsilonB + \eT(X_{T_2} | X_{T_3})}, e^{\epsilonA+\epsilonB}\}$.
\end{proof}
\begin{proof}(of Theorem~\ref{thm:para_mqm})
Denote the noises added by \mqm\ for Alice and Bob by $\ZA,\ZB$ respectively. Consider any secret pair of the form $(\XA_i=a,\XA_i=b)$. We want to upper bound the following ratio for any $\wA, \wB, \theta, i, a,b$.
\begin{align*}
\frac	{P(\FA(\XA)+\ZA=\wA, \FB(\XB)+\ZB=\wB | \XA_i=a, \theta)}
		{P(\FA(\XA)+\ZA=\wA, \FB(\XB)+\ZB=\wB | \XA_i=b, \theta)}.		
\end{align*}

By assumption, there exists some $X_j \in \XA$ whose active Markov Quilt is $X_{Q,j} = \{X_{j-a}, X_{j+b}\}$ with corresponding $X_{R,j}$ and $X_{N,j}$; and we have $\sigma_{\max}^\theta \geq \sigma_j^\theta = \card{X_{N,j}} / (\epsilon - \et(X_{Q,j} | X_j))$.

The main idea of the proof is that we can ``borrow" the Markov Quilt of $X_j$ as the Markov Quilt for any $X_i \in \XA$ because doing so will not increase the noise scale $\sigma_{\max}^\theta$. 
There are three cases:

\begin{enumerate}
\item If $i-a \geq 1$ and $i+b\leq T_2$, then let $X_Q = \{X_{i-a}, X_{i+b}\}$ (we omit the subscript $i$ for simplicity) with corresponding $X_R = \{X_k\}_{1\leq k<i-a \text{ or } i+b<k\leq T_2}$ and $X_N = \\ \{X_k\}_{i-a<k<i+b}$.
\item If $i-a \geq 1$ and $i+b > T_2$, then let $X_Q = \{X_{i-a}, X_{i+b}\}$ with corresponding $X_R = \{X_k\}_{1\leq k<i-a}$ and $X_N = \{X_k\}_{i-a<k<i+b}$.
\item If  $i+b\leq T_2$ and $i-a \leq 0$,  then let $X_Q = \{X_{i+b}\}$ with corresponding $X_N = \{X_k\}_{1\leq k<i+b}$ and $X_R = \{X_k\}_{i+b<k\leq T_2}$. 
\end{enumerate}

Notice that when \effect\ is approximated with \\ Lemma~\ref{lem:effect_approx}, for any $i$ and $j$, we have $e_\theta(\{X_{i-a},X_{i+a}\} | X_i) = e_\theta(\{X_{j-a},X_{j+a}\} | X_j)$, i.e., the \effect\ is only affected by the relative distance between $X_i$ and its Markov Quilt.

Therefore, in the first two cases, we have $\sigma_i^\theta(X_Q) = \sigma_j^\theta(X_{Q,j})$ since the \effect\ and the size of nearby nodes are the same; in the last case, since $\et(X_Q | X_i) \leq \et(X_{Q,j} | X_j)$ and $\card{X_N} \leq \card{X_{N,j}}$, we have $\sigma_i^\theta(X_Q) \leq \sigma_j^\theta(X_{Q,j})$. Therefore we know that $\Lap(\sigma_{\max}^\theta)$ suffices to protect $X_i$.

Let $X_{\RQ} = X_R \cup X_Q$. 
We can split $X_{\RQ}$ into two parts, $X_{\inner}=\{X_j \in X_\RQ, j>i\}$ which is closer to the middle of $\XA$ and $\XB$, and $X_{\out}=\{X_j \in X_\RQ, j<i\}$ which is closer to the boundary of the Markov chain.
For the three cases respectively, we have
\begin{enumerate}
\item $X_{\inner} = \{X_k\}_{i+b\leq k\leq T_2}$ and $X_{\out} =  \{X_k\}_{1\leq k\leq i-a}$.
\item $X_{\inner} = \{X_{i+b}\}$ and $X_{\out} =  \{X_k\}_{1\leq k\leq i-a}$.
\item $X_{\inner} = \{X_k\}_{i+b\leq k\leq T_2}$ and $X_{\out} =  \emptyset$.
\end{enumerate}

By assumption, $\XA$ and $\XB$ are far enough, and thus $X_{i+b} \notin \XB$ and $X_{\inner} \cap \XB = \emptyset$.

Then we have
\begin{align*}
&		P(\FA(\XA)+\ZA=\wA, \FB(\XB)+\ZB=\wB | X_i=a, \theta)\\
=&	\int_{} P(\FA(\XA)+\ZA=\wA, \FB(\XB)+\ZB=\wB,\\ &X_{\out}=x_{\out},X_{\inner}=x_{\inner}| X_i=a, \theta) d{x_{\out}} d{x_{\inner}}\\
=&	\int_{} P(\FA(\XA)+\ZA=\wA, \FB(\XB)+\ZB=\wB, X_{\out}=\\
& x_{\out},X_i=a|X_{\inner}=x_{\inner} , \theta)P(X_{\inner}=x_{\inner}|\theta) / \\ 
& P(X_i=a|\theta)d{x_{\out}} d{x_{\inner}}\\
=&	\int_{} P(\FA(\XA)+\ZA=\wA, X_{\out}=x_{\out},X_i=a|X_{\inner}=\\
& x_{\inner}, \theta)P(\FB(\XB)+\ZB=\wB|X_{\inner}=x_{\inner}, \theta) \\
& P(X_{\inner}=x_{\inner}|\theta) / 
 P(X_i=a|\theta)d{x_{\out}} d{x_{\inner}}\\
=&	\int_{} P(\FA(\XA)+\ZA=\wA, X_{\out}=x_{\out},X_{\inner}=x_{\inner}|X_i=a,\\
& \theta)P(\FB(\XB)+\ZB=\wB|X_{\inner}=x_{\inner}, \theta)d{x_{\out}} d{x_{\inner}},
\end{align*}
where the third equality follows because $X_{\inner}$ separates $\XB$ with $\XA$. 

By Lemma \ref{lem:fx and xRQ}, for any $a$,$b$ and and $x_\RQ$, 
\begin{align*}
\frac{P(\FA(\XA)+\ZA=\wA, X_{\RQ}=x_{\RQ}|X_i=a, \theta)}{P(\FA(\XA)+\ZA=\wA, X_{\RQ}=x_{\RQ}|X_i=b, \theta)} \leq e^\epsilonA.
\end{align*}
Therefore for any ${X_{\out},X_{\inner}} $ we have
\begin{align*}
&\frac{\splitfrac{P(\FA(\XA)+\ZA=\wA, X_{\out}=x_{\out},X_{\inner}=x_{\inner}|X_i=a, \theta)}{P(\FB(\XB)+\ZB=\wB|X_{\inner}=x_{\inner}, \theta)}}
	 {\splitfrac{P(\FA(\XA)+\ZA=\wA, X_{\out}=x_{\out},X_{\inner}=x_{\inner}|X_i=b, \theta)}{P(\FB(\XB)+\ZB=\wB|X_{\inner}=x_{\inner}, \theta)}} \\
& \leq e^\epsilonA,
\end{align*}
and therefore
\begin{align*}
&	\frac{P(\FA(\XA)+\ZA=\wA, \FB(\XB)+\ZB=\wB | X_i=a, \theta)}{P(\FA(\XA)+\ZA=\wA, \FB(\XB)+\ZB=\wB | X_i=b, \theta)} \\
\leq&	\max_{x_{\out},x_{\inner}} \frac{\splitfrac{P(\FA(\XA)+\ZA=\wA, X_{\out}=x_{\out},X_{\inner}=x_{\inner}|}{X_i=a, \theta)P(\FB(\XB)+\ZB=\wB|X_{\inner}=x_{\inner}, \theta)}}
		 {\splitfrac{P(\FA(\XA)+\ZA=\wA, X_{\out}=x_{\out},X_{\inner}=x_{\inner}|}{X_i=b, \theta)P(\FB(\XB)+\ZB=\wB|X_{\inner}=x_{\inner}, \theta)}}\\
\leq& e^\epsilonA.
\end{align*}
When the secret pair is of the form $(\XB_i=a,\XB_i=b)$, similar argument applies and the bound is $e^\epsilonB$.

Therefore for Pufferfish parameter $(\SA\cup\SB,\QA\cup\QB,\Theta)$, $(\MQM(X,\FA,\epsilonA,(\SA,\QA,\Theta)), \MQM(\XB,\FB,\epsilonB, \\ (\SB,\QB,\Theta)))$ guarantees $\max(\epsilonA,\epsilonB)$-Pufferfish Privacy.
\end{proof}

\subsection{Proofs for Sequential Composition Results}
\begin{proof}(of Theorem~\ref{thm:sequential_any_pf_counterex})
Consider a Markov chain with two nodes: $X=\{X_1,X_2\}$ and state space $\{0,1\}$. Let $\Pr(X_2=1|X_1=1)=p, \Pr(X_2=1|X_1=0)=q$, i.e., the transition matrix is $[1-q, q; 1-p, p]$. Suppose the set of secret pairs is 
$\calQ=\{(X_1=0,X_1=1)\}$.
Let $\calF$ be the summation function, i.e., $\calF(X)=X_1+X_2$. Suppose mechanism $\calM$ outputs $\calF(X)+Z$ where $Z = \Lap(1)$.

Denote 
\begin{align*}
S(w) =& \frac{p(\calF(X)+Z=w | X_1=1)}{p(\calF(X)+Z=w | X_1=0)}\\
D(w,w') =& \frac{{p(\calF(X)+Z=w,\calF(X)+Z=w' |}{ X_1=1)}}{{p(\calF(X)+Z=w,\calF(X)+Z=w'|}{X_1=0)}}.
\end{align*}

So $\calM(X)$ guarantees $\log \max_{w}\{S(w),1/S(w)\}$-Pufferfish privacy and $(\calM(X),\calM(X))$ guarantees $\log \max_{w,w'}\{D(w,w'),1/D(w,w')\}$-Pufferfish privacy.

We aim at proving that $2\log(S) < \log(D)$, equivalently, $S^2 < D$.

Let $w$ denote the event that $\calF(X)+Z=w$, and $\sigma$ denote the probability distribution of $Z$, i.e., $\Lap(1)$.
We have
\begin{align*}
S(w)%
=&	\frac{p(w | X_1=1,X_2=1)p(X_2=1|X_1=1)}{p(w | X_1=0,X_2=1)p(X_2=1|X_1=0)} \\&\frac{ + p(w | X_1=1,X_2=0)p(X_2=0|X_1=1)}{ + p(w | X_1=0,X_2=0)p(X_2=0|X_1=0)}\\
=&	\frac{\sigma(w-2)p + \sigma(w-1)(1-p)}{\sigma(w-1)q + \sigma(w)(1-q)}
\end{align*}
and
\begin{align*}
&D(w,w')=\\
&	\frac{\sigma(w-2)\sigma(w'-2)p + \sigma(w-1)\sigma(w'-1)(1-p)}{\sigma(w-1)\sigma(w'-1)q + \sigma(w)\sigma(w')(1-q)}.
\end{align*}
\rmk{In the ratio of probabilities conditioned on $X_2$, assuming $\Pr(X_1=0)=\Pr(X_1=1)$, $p$ becomes $p/(p+q) < p$, and $q$ becomes $(1-p)/(2-p-q) > q$; and thus the final possible values will only be smaller.}\\

To simplify the analysis, we consider
\begin{align*}
D(w)=D(w,w)=	\frac{\sigma(2w-4)p + \sigma(2w-2)(1-p)}{\sigma(2w-2)q + \sigma(2w)(1-q)}
\end{align*}
which equals to $\max_{w'=w}\{D(w,w'), 1/D(w,w')\}$, and thus is a lower bound of  $\max_{w,w'}\{D(w,w'), 1/D(w,w')\}$.

There are in total 3 possible values of $S(w)$ under different $w$. 
Case 1: $w\in(-\inf, 0]$.
\begin{align*}
S(w)=\frac{e^{w-2}p + e^{w-1}(1-p)}{e^{w-1}q + e^{w}(1-q)}=\frac{p + e^{1}(1-p)}{e^{1}q + e^{2}(1-q)}
\end{align*}
Case 2: $w\in[0, 1]$.
\begin{align*}
S(w)=&\frac{e^{w-2}p + e^{w-1}(1-p)}{e^{w-1}q + e^{-w}(1-q)}
=\frac{p + e^{1}(1-p)}{e^{1}q + e^{2-2w}(1-q)}\\
=&\frac{p + e^{1}(1-p)}{e^{1}q + e^{2}(1-q)} \text{ or } \frac{p + e^{1}(1-p)}{e^{1}q + (1-q)}
\end{align*}
Case 3: $w\in[1, 2]$.
\begin{align*}
S(w)
=&\frac{e^{w-2}p + e^{-w+1}(1-p)}{e^{-w+1}q + e^{-w}(1-q)}
=\frac{e^{2w-2}p + e^{1}(1-p)}{e^{1}q + (1-q)}\\
=&\frac{p + e^{1}(1-p)}{e^{1}q + (1-q)} \text{ or } \frac{e^{2}p + e^{1}(1-p)}{e^{1}q + (1-q)}
\end{align*}
Case 4: $w\in[2, \inf)$.
\begin{align*}
S(w)
=\frac{e^{-w+2}p + e^{-w+1}(1-p)}{e^{-w+1}q + e^{-w}(1-q)}
=\frac{e^{2}p + e^{1}(1-p)}{e^{1}q + (1-q)}.
\end{align*}
So in total $S(w)$ has 3 possible values:
\begin{align*}
\frac{p + e^{1}(1-p)}{e^{1}q + e^{2}(1-q)} \leq \frac{p + e^{1}(1-p)}{e^{1}q + (1-q)} \leq \frac{e^{2}p + e^{1}(1-p)}{e^{1}q + (1-q)},
\end{align*}
and because of the inequality relation, possible values of $\max_{w}\{S(w),1/S(w)\}$ are
\begin{align*}
& \frac{e^{1}q + e^{2}(1-q)}{p + e^{1}(1-p)} = e\frac{q + e(1-q)}{p + e(1-p)}, \\
& \frac{e^{2}p + e^{1}(1-p)}{e^{1}q + (1-q)} = e\frac{ep + (1-p)}{eq + (1-q)}.
\end{align*}
Similarly, possible values of $\max_{w}\{D(w),1/D(w)\}$ are:
\begin{align*}
e^{2}\frac{q + e^{2}(1-q)}{p + e^{2}(1-p)},  e^{2}\frac{e^{2}p + (1-p)}{e^{2}q + (1-q)}.
\end{align*}
Note that the first possible value of $S(w)$ is larger whenever the first value of $D(w)$ is larger. (Both are quadratic centered at $1/2$.)\\
Assume $p=0.9, q=0.01$. Then 
\begin{align*}
&(\max_{w}\{S(w),1/S(w)\})^2 = e^2 \max\{5.3132, 6.2672\} \\
&\max_{w,w'}\{D(w,w'),1/D(w,w')\} \geq \max_{w}\{D(w),1/D(w)\} \\	=& e^2 \max\{4.4695, 6.3448\}
\end{align*}
Therefore when $\calM(X)$ guarantees $\epsilon$-differential privacy, \\ $(\calM(X),\calM(X))$ cannot guarantee $2\epsilon$-differential privacy.
\end{proof}
\rmk{Considering binary states. Pufferfish composes perfectly when every node is independent ($p=q$) or completely dependent ($p=1,q=0$). So there is probably some term measuring $\min\{p-q,p-(1-q)\}$.
}
\begin{proof}(of Theorem~\ref{thm:sequential_any_pf})
By the assumption of the theorem, for all $s_i \in \calS, \theta \in \Theta$ and all $w_1, w_2$, we have
\begin{align*}
e^{-E} \leq \frac{P(\MA(X) = w_1, \MB(X) = w_2 | s_i, \theta)}{P(\MA(X) = w_1 | s_i, \theta) P(\MB(X) = w_2 | s_i, \theta)} \leq e^{E},
\end{align*}
which is equivalent to 
\begin{align*}
e^{-E} \leq \frac{P(\MB(X) = w_2 | \MA(X) = w_1 , s_i, \theta)}{P(\MB(X) = w_2 | s_i, \theta)} \leq e^{E}.
\end{align*}
For any $(s_i, s_j) \in \calQ$ and any $w_1, w_2$, we have
\begin{align*}
&	\log \frac{P(\MA(X) = w_1, \MB(X) = w_2 | s_i, \theta)}{P(\MA(X) = w_1, \MB(X) = w_2 | s_j, \theta)}\\
=&	\log \frac{P(\MA(X) = w_1 | s_i, \theta)}{P(\MA(X) = w_1 | s_j, \theta)} + \\
 & \log \frac{P(\MB(X) = w_2 | \MA(X) = w_1, s_i, \theta)}{P(\MB(X) = w_2 | \MA(X) = w_1, s_j, \theta)}\\
\leq&	\log \frac{P(\MA(X) = w_1 | s_i, \theta)}{P(\MA(X) = w_1 | s_j, \theta)} + \\
 & \log \frac{P(\MB(X) = w_2 | s_i, \theta)}{P(\MB(X) = w_2 | s_j, \theta)} \frac{e^{E}}{e^{-E}}\\
\leq&	\epsilonA + \epsilonB + 2E.
\end{align*}
\end{proof}
\begin{proof}(of Theorem~\ref{thm:sequential_mqm_chain})
For simplicity we consider $K=2$ first.

Consider any secret pair $(X_i=\a,X_i=\b)$ and any $\theta\in\Theta$. Let $X_Q^1 = \{X_{i-\aone}, X_{i+\bone}\}$ be the active Markov Quilt of $X_i$ in the first publication 
 (with $X_N^1 = \{X_k\}_{i-\aone<k<i+\bone}$, $X_R^1 = \{X_k\}_{1\leq k<i-\aone \text{ or } i+\bone<k\leq T}$)
, and $X_Q^2 = \{X_{i-\atwo}, X_{i+\btwo}\}$ be that in the second
 (with $X_N^2 = \{X_k\}_{i-\atwo<k<i+\btwo}$, $X_R^2 = \{X_k\}_{1\leq k<i-\atwo \text{ or } i+\btwo<k\leq T}$).
 Denote $X_R^j\cup X_Q^j$ as $X_\RQ^j=$ for $j=1,2$.
 
 Let $X_\RQ^{} = \cup_{j=1}^{2} X_\RQ^j$. Let $X_N^{} = \cap_{j=1}^{2} X_N^j$, which is guaranteed to be non-empty since it contains at least $X_i$. 
 Let $X_Q = \{X_{i-\min(\aone,\atwo)}, X_{i+\min(\bone,\btwo)}\}$, i.e., we pick from $X_{i-\aone}$ and $X_{i-\atwo}$ (and also $X_{i+\bone}$ and $X_{i+\btwo}$) the ones that are closer to $X_i$. Note that this is a valid Markov Quilt of $X_i$, with corresponding nearby set $X_N^{}$ and remote set $X_\RQ^{}\backslash X_Q$.

Let $Z_1,Z_2$ denote the Laplace noises added by \mqm\ for the two releases respectively.  
For simplicity, we omit the $\theta$ term in the probabilities and assume all $X_j$s are distributed according to $\theta$.
Then for any $w_1,w_2$, we have
\begin{align}\label{eqn:comp serial main}
&\frac{p(\calF_1(X)+Z_1=w_1, \calF_2(X)+Z_2=w_2 | X_i=\a)}{p(\calF_1(X)+Z_1=w_1, \calF_2(X)+Z_2=w_2 | X_i=\b)} \nonumber\\
=&\frac	
        {\splitfrac{\int_{} p(\calF_1(X)+Z_1=w_1, \calF_2(X)+Z_2=w_2,}{X_\RQ^{}=x_\RQ^{}| X_i=\a) d{x_\RQ^{}}}}
		{\splitfrac{\int_{} p(\calF_1(X)+Z_1=w_1, \calF_2(X)+Z_2=w_2,}{X_\RQ^{}=x_\RQ^{}| X_i=\b) d{x_\RQ^{}}}} \nonumber\\
=&\frac	{\splitfrac{\int_{} p(\calF_1(X)+Z_1=w_1, \calF_2(X)+Z_2=w_2|X_\RQ^{}=}{x_\RQ^{}, X_i=\a)           p(X_\RQ^{}=x_\RQ^{}|X_i=\a) d{x_\RQ^{}}}}
		{\splitfrac{\int_{} p(\calF_1(X)+Z_1=w_1, \calF_2(X)+Z_2=w_2|X_\RQ^{}=}{x_\RQ^{}, X_i=\b) p(X_\RQ^{}=x_\RQ^{}|X_i=\b) d{x_\RQ^{}}}} \nonumber\\
\leq&\max_{x_\RQ^{}}\frac	
        {\splitfrac{p(\calF_1(X)+Z_1=w_1, \calF_2(X)+Z_2=w_2|X_\RQ^{}=}{x_\RQ^{}, X_i=\a) p(X_\RQ^{}=x_\RQ^{}|X_i=\a)}}
		{\splitfrac{p(\calF_1(X)+Z_1=w_1, \calF_2(X)+Z_2=w_2|X_\RQ^{}=}{x_\RQ^{}, X_i=\b) p(X_\RQ^{}=x_\RQ^{}|X_i=\b)}}.
\end{align}
First, consider the first ratio in \eqref{eqn:comp serial main}. Let $X_{N\backslash\{i\}}^{} = X_N^{}\backslash X_{i}$ denote all ``nearby" nodes except for $X_i$. 
Let $\calF_*(x_i,{x_{N\backslash\{i\}}^{}},x_\RQ^{})$ denote the function value of $\calF_*$ when $X_i=x_i$, ${X_{N\backslash\{i\}}^{}}={X_{N\backslash\{i\}}^{}={x_{N\backslash\{i\}}^{}}}$ and $X_\RQ^{}=x_\RQ^{}$.
We have for $x_i = \a$ or $\b$,
\begin{align*}
&	p(\calF_1(X)+Z_1=w_1, \calF_2(X)+Z_2=w_2| \\  &X_\RQ^{}=x_\RQ^{}, X_i=x_i) \\
=	\int_{} \ & p(\calF_1(X)+Z_1=w_1, \calF_2(X)+Z_2=w_2, \\&{X_{N\backslash\{i\}}^{}={x_{N\backslash\{i\}}^{}}}| X_\RQ^{}=x_\RQ^{}, X_i=x_i) d {x_{N\backslash\{i\}}^{}}\\
=	\int_{} \ & p(\calF_1(X)+Z_1=w_1, \calF_2(X)+Z_2=w_2|\\&{X_{N\backslash\{i\}}^{}={x_{N\backslash\{i\}}^{}}}, X_\RQ^{}=x_\RQ^{}, X_i=x_i) \\&p({X_{N\backslash\{i\}}^{}={x_{N\backslash\{i\}}^{}}}|X_\RQ^{}=x_\RQ^{},X_i=x_i) d {x_{N\backslash\{i\}}^{}}\\
=	\int_{} \ & p(Z_1=w_1-\calF_1(a,{x_{N\backslash\{i\}}^{}},x_\RQ^{})) \\&p(Z_2=w_2-\calF_2(a,{x_{N\backslash\{i\}}^{}},x_\RQ^{})) \\&p({X_{N\backslash\{i\}}^{}={x_{N\backslash\{i\}}^{}}}|X_\RQ^{}=x_\RQ^{},X_i=x_i) d {x_{N\backslash\{i\}}^{}},
\end{align*}
where the last equality follows because $Z_1$ and $Z_2$ are independent given the value of $\calF_1, \calF_2$. 

Now we can consider the ratio of the above formula at $X_i=\a$ and $\b$.
Notice that $\calF_1,\calF_2$ can change by at most $\card{X_N^{}} \leq  \min\{\card{X_N^{1}}, \card{X_N^{2}}\}$ when $X_\Ni^{}$ and $X_i$ change. Therefore we have
\begin{align*}
&\max_{{x_{N\backslash\{i\}}^{}}, {x'_{N\backslash\{i\}}}}
\frac {p(Z_1=w_1-\calF_1(a,{x_{N\backslash\{i\}}^{}},x_\RQ^{}))} {p(Z_1=w_1-\calF_1(b,{x'_{N\backslash\{i\}}},x_\RQ^{}))} \\
\leq & e^{(\epsilon_1 - \et(X_Q^1|X_i))/\card{X_N^1} \times \card{X_N}} \leq e^{\epsilon_1 - \et(X_Q^1|X_i)}, \\
&\max_{{x_{N\backslash\{i\}}^{}}, {x'_{N\backslash\{i\}}}}
\frac {p(Z_2=w_2-\calF_2(a,{x_{N\backslash\{i\}}^{}},x_\RQ^{}))} {p(Z_2=w_2-\calF_2(b,{x'_{N\backslash\{i\}}},x_\RQ^{}))} \\
\leq & e^{(\epsilon_2 - \et(X_Q^2|X_i))/\card{X_N^2} \times \card{X_N}} \leq e^{\epsilon_2 - \et(X_Q^2|X_i)}.
\end{align*}
Moreover, for any $x_i$, $\int p({X_{N\backslash\{i\}}^{}={x_{N\backslash\{i\}}^{}}}|X_\RQ^{}=x_\RQ^{},X_i=x_i) d {x_{N\backslash\{i\}}^{}}$ equals to $1$. Therefore the first ratio in \eqref{eqn:comp serial main} is upper bounded by 
\begin{align}\label{eqn:comp serial 1st}
e^{\epsilon_1 + \epsilon_2 - \et(X_Q^1|X_i) - \et(X_Q^2|X_i)}.
\end{align}
Then we consider the second ratio in \eqref{eqn:comp serial main}.
\begin{align*}
&\frac{p(X_\RQ^{}=x_\RQ^{}|X_i=\a)}{p(X_\RQ^{}=x_\RQ^{}|X_i=\b)} \\
=&	\frac{p(X_R=x_R|X_Q=x_Q,X_i=\a)p(X_Q=x_Q|X_i=\a)}{p(X_R=x_R|X_Q=x_Q,X_i=\b)p(X_Q=x_Q|X_i=\b)} \\
=&	\frac{p(X_R=x_R|X_Q=x_Q)p(X_Q=x_Q|X_i=\a)}{p(X_R=x_R|X_Q=x_Q)p(X_Q=x_Q|X_i=\b)} \\
=&	\frac{p(X_Q=x_Q|X_i=\a)}{p(X_Q=x_Q|X_i=\b)} 
\leq	e^{\et(X_Q|X_i)},
\end{align*}
where the second equality comes from the fact that $X_Q$ is a valid Markov Quilt of $X_i$, and the last inequality follows from the definition of \effect.\\
Now we show that 
\begin{align}\label{eqn:comp serial together}
\et(X_Q^1|X_i) + \et(X_Q^2|X_i) 
\geq \et(X_Q | X_i).
\end{align} 
If $X_Q$ is equal to one of $X_Q^1$ and $X_Q^2$, the inequality hold trivially.
Otherwise, according to Lemma~\ref{lem:max influence ineq}, we have
\begin{align*}
 & \et(\{X_a,X_b\} | X_i) + \et(\{X_\atwo,X_\btwo\} | X_i) \\
\geq &  
 \et(X_{\max(\aone,\atwo)} | X_i) + \et(X_{\min(\bone,\btwo)} | X_i) \\
\geq & \et(\{X_{\max(\aone,\atwo)}, X_{\min(\bone,\btwo)}\} | X_i).
\end{align*}
Combining with \eqref{eqn:comp serial 1st}, we know that \eqref{eqn:comp serial main} is upper bounded by 
\begin{align*}
e^{\epsilon_1 + \epsilon_2 - \et(X_Q^1|X_i) - \et(X_Q^2|X_i)} e^{\et(X_Q|X_i)} \leq e^{\epsilon_1+\epsilon_2}.
\end{align*}
{The same analysis holds for any number of compositions. Suppose we have $\{X_Q^j = \{X_{a^j}, X_{b^j}\}\}_{j=1}^k$ as the quilts used in all the $k$ runs of \mqm. We can set $X_N$ to be the intersubsection of all $X_N^j$, $X_\RQ$ be the union of all $X_\RQ^j$, and $X_Q = \{X_{i-\min(\{a^j\})}, X_{i+\min(\{b^j\})}\}$. Then we still have \eqref{eqn:comp serial together}; basically, apart from the two $X_Q^*$'s each contributing to one side of $X_Q$, other $\et(X_Q^* | X_i)$'s are not used. So the bound is still pretty loose; the more runs we have, the looser the bound is.}
\end{proof}

\subsection{Other Lemmas}
\begin{lemma}\label{lem:max influence ineq}
Let $X_S$ and $X_R$ be two sets of nodes in a Bayesian network such that $X_S \subseteq X_R$.
For any $\Theta$, we have
\begin{align*}
\eT(X_{S} | X_i) \leq \eT(X_{R} | X_i).
\end{align*}
\end{lemma}
\begin{proof}
Let $T = R \backslash S$,for any $\theta\in\Theta$ we have
\begin{align*}
&\exp(e_{\theta}(X_R | X_i))\\
=&	\max_{x_R, x_i, x_i'} \frac{p(X_R=x_R | X_i=x_i,\theta)}{p(X_R=x_R | X_i=x_i',\theta)} \\
=&	\max_{x_S, x_T, x_i, x_i'} \frac{p(X_S=x_S, X_T=x_T | X_i=x_i,\theta)}{p(X_S=x_S, X_T=x_T | X_i=x_i',\theta)}.
\end{align*}
Then we have
\begin{align*}
&e_{\theta}(X_{S} | X_i,\theta)\\
=&	\max_{x_S, x_i, x_i'} \frac{p(X_S=x_S | X_i=x_i,\theta)}{p(X_S=x_S | X_i=x_i',\theta)} \\
=&	\max_{x_S, x_i, x_i'} \frac{\sum_{X_T=x_T} p(X_S=x_S, X_T=x_T | X_i=x_i,\theta)}{\sum_{x_T} p(X_S=x_S, X_T=x_T | X_i=x_i',\theta)} \\
\leq&	\max_{x_S, x_i, x_i'} \frac{\sum_{x_T} e^{e_{\theta}(X_R | X_i)} p(X_S=x_S, X_T=x_T | X_i=x_i',\theta)}{\sum_{x_T} p(X_S=x_S, X_T=x_T | X_i=x_i',\theta)} \\
=&	\exp(e_{\theta}(X_R | X_i)),
\end{align*}
where the inequality is from the definition of \effect.
Since this holds for all $\theta\in\Theta$, we have $\eT(X_{S} | X_i,\theta) \leq \eT(X_{R} | X_i,\theta)$.
\end{proof}

Here we also prove a useful lemma which is similar to the privacy guarantee of \mqm\ proved in \cite{ourpufferfish}.

\begin{lemma}\label{lem:fx and xRQ}
For any secret pair $(X_i=a, X_i=b) \in \calQ$ and any $\theta \in \Theta$, let $X_Q$ be the Markov Quilt for $X_i$ which has the minimum score $\sigma(X_Q)$, and suppose that deleting $X_Q$ breaks up the underlying Bayesian network into $X_N$ and $X_R$ where $X_i \in X_N$. Then for any $w$ and any realization $x_\RQ$ of $X_\RQ$, 
\begin{align*}
&	\frac{ p(F(X) + \sigma_{\max}\cdot Z = w, X_\RQ=x_\RQ | X_i = a, \theta) }{p(F(X) + \sigma_{\max} \cdot Z = w , X_\RQ=x_\RQ| X_i = b, \theta)} \leq e^\epsilon.
\end{align*}
\end{lemma}
\begin{proof}
Pick a secret pair $(X_i=a, X_i=b) \in \calQ$ and any $\theta \in \Theta$. Let $X_Q$ be the Markov Quilt for $X_i$ which has the minimum score $\sigma(X_Q)$, and suppose that deleting $X_Q$ breaks up the underlying Bayesian network into $X_N$ and $X_R$ where $X_i \in X_N$.\\
For any $w$ and any realization $x_\RQ$ of $X_\RQ$, we can write
\begin{align} \label{eqn:mqprivacy}
&	\frac{ p(F(X) + \sigma_{\max}\cdot Z = w, X_\RQ=x_\RQ | X_i = a, \theta) }{p(F(X) + \sigma_{\max} \cdot Z = w , X_\RQ=x_\RQ| X_i = b, \theta)} \nonumber \\
=&	\frac{p(F(X) + \sigma_{\max} Z = w | X_i = a, X_{R \cup Q} = x_{R \cup Q}, \theta) }
		 {p(F(X) + \sigma_{\max} Z = w | X_i = b, X_{R \cup Q} = x_{R \cup Q}, \theta) } \cdot \nonumber\\
 & \frac{p(X_{R \cup Q} = x_{R \cup Q} | X_i = a, \theta)}{p(X_{R \cup Q} = x_{R \cup Q} | X_i = b, \theta)}.
\end{align}
Consider the first ratio of~\eqref{eqn:mqprivacy},
\begin{equation*}
\frac{p(F(X) + \sigma_{\max} Z = w | X_i = a, X_{R \cup Q} = x_{R \cup Q}, \theta)} {p(F(X) + \sigma_{\max} Z = w | X_i = b, X_{R \cup Q} = x_{R \cup Q}, \theta)} .
\end{equation*}
Since $F$ is $1$-Lipschitz, when $X_{R \cup Q}$ is fixed, $F(X)$ can vary by at most $\card{X_N}$ (potentially when all the variables in $X_N$ change values). Since $\sigma_{\max} \geq \frac{\card{X_N}}{\epsilon - \et(X_Q | X_i)}$ for any $X_i$ with its best Markov Quilt $X_Q$, and $Z\sim\Lap(1)$, we know that the above ratio is upper bounded by 
\[ e^{\epsilon - \et(X_Q | X_i)}. \]
Then consider the second part of~\eqref{eqn:mqprivacy}. We have
\begin{align*}
& \frac{p(X_{R \cup Q} = x_{R \cup Q} | X_i = a, \theta)}{p(X_{R \cup Q} = x_{R \cup Q} | X_i = b, \theta)} \\
= & \frac{p(X_R = x_R | X_Q = x_Q, X_i = a, \theta) p(X_Q = x_Q | X_i = a, \theta)}{ p(X_R = x_R| X_Q = x_Q, X_i = b, \theta) p(X_Q = x_Q | X_i = b, \theta)}.    
\end{align*}
Since $X_Q$ is a Markov Quilt for $X_i$ and $X_i \notin X_R$, we have $p(X_R | X_Q, X_i = a, \theta) = p(X_R | X_Q, X_i = b, \theta)$. Moreover, by definition of \effect, $\frac{p(X_Q = x_Q | X_i = a, \theta)}{p(X_Q = x_Q | X_i = b, \theta)} \leq e^{\et(X_Q | X_i)}$. Therefore the above ratio is upper bounded by $$e^{\et(X_Q | X_i)}.$$
Combining the two ratios together, we can conclude that for any $w$ and any secret pair $(s^i_a, s^i_b)$,
\[ \frac{ p(F(X) + \sigma_{\max}\cdot Z = w, X_\RQ=x_\RQ | X_i = a, \theta) }{p(F(X) + \sigma_{\max} \cdot Z = w , X_\RQ=x_\RQ| X_i = b, \theta)} \leq e^{\epsilon}. \]
\end{proof}

\end{document}